\documentclass[runningheads]{llncs}
\usepackage{graphicx}
\usepackage[utf8]{inputenc} 
\usepackage[T1]{fontenc}    
\usepackage{hyperref}       
\usepackage{url}            
\usepackage{booktabs}       
\usepackage{amsfonts}       
\usepackage{nicefrac}       
\usepackage{microtype}      
\usepackage{amsmath}
\usepackage{subfigure}
\usepackage{rotating}
\usepackage{comment}
\usepackage{siunitx}
\usepackage[]{hyperref}
\usepackage{graphicx}
\usepackage{tikz}
\usepackage{float}

\usetikzlibrary{external}
\renewcommand{\xi}[1][i]{\mathbf{x}^{(#1)}}                                          

\usepackage{pxfonts}
\usepackage{mathtools}

\newcommand{\loss}{\mathcal{L}}

\newcommand{\idp}{\Perp}
\newcommand{\dpd}{\not \idp}

\begin{document}
\title{Decomposition of Global Feature Importance into Direct and Associative Components (DEDACT)\thanks{This project is funded by the German Federal Ministry of Education and Research (BMBF) under Grant No. 01IS18036A. The authors of this work take full responsibility for it's content.}}
\titlerunning{DEDACT}
%
\author{Gunnar König\inst{1,2}\orcidID{0000-0001-6141-4942} \and
Timo Freiesleben\inst{1}\orcidID{0000-0003-1338-3293} \and\\
Bernd Bischl\inst{1}\orcidID{0000-0001-6002-6980} \and
Giuseppe Casalicchio\inst{1}\orcidID{0000-0001-5324-5966} \and
Moritz Grosse-Wentrup\inst{2}\orcidID{0000-0001-9787-2291}
}

\authorrunning{G. König et al.}
%
\institute{Ludwig-Maximilian University Munich,  Germany \and 
  University of Vienna, Austria}

\maketitle              
\begin{abstract}
Global model-agnostic feature importance measures either quantify whether features are directly used for a model’s predictions (direct importance) or whether they contain prediction-relevant information (associative importance).
Direct importance provides causal insight into the model's mechanism, yet it fails to expose the leakage of information from associated but not directly used variables. In contrast, associative importance exposes information leakage but does not provide causal insight into the model's mechanism.
We introduce DEDACT – a framework to decompose well-established direct and associative importance measures into their respective associative and direct components. DEDACT provides insight into both the sources of prediction-relevant information in the data and the direct and indirect feature pathways by which the information enters the model. We demonstrate the method's usefulness on simulated examples.
\end{abstract}

\section{Introduction}
\label{sec:introduction}

Black-box machine learning models are increasingly deployed in high stakes environments such as hiring, criminal justice or medical diagnosis \cite{chalfin2016productivity,berk2012criminal,rajkomar2019machine}. Despite providing accurate predictions in a test environment, their application can be harmful as their predictions may rely on spuriously correlated variables or protected attributes. In such cases, models can discriminate disadvantaged groups \cite{wexler_rebecca_when_2017} or make inaccurate predictions outside the test environment \cite{zech_variable_2018}.\\
A range of Interpretable Machine Learning (IML) techniques provide insight into the model's decision-making by quantifying the importance of features for the model's performance \cite{molnar2019}.
Measures of feature importance differ in terms of their consideration of the dependence structure in the underlying dataset. While some methods only quantify the direct effect of features others also deem variables important that contain prediction-relevant information, irrespective of their direct utilization by the model. We will refer to the former as direct importance measures and to the latter as associative importance measures.\\
While direct importance measures allow causal insight into the model's mechanism, they fail to expose leakage of information from associated but not directly used variables\cite{Adler2018}. In contrast, associative importance methods expose leakage of information but do not provide insight into the model's mechanism\cite{janzing2020feature}. However, 
even in combination, direct and associative importance cannot provide both insight into the sources of prediction-relevant information \textit{and} the direct and indirect feature pathways by which the information enters the model.\\
Consider a simple illustrative example (Figure \ref{fig:example-intro}): Features \textit{zip code} and \textit{job experience} are directly used by the model to compute the prediction $\hat{Y}$. The feature \textit{ethnicity} is associated with the variable \textit{zip code} (blue arrow) and therefore with $\hat{Y}$; however, \textit{ethnicity} is not directly used by the model.
\begin{figure}[H]
    \centering
    \begin{tikzpicture}[thick, scale=0.7, every node/.style={scale=.6, line width=0.25mm, black, fill=white}]
    \usetikzlibrary{shapes}
		\node[draw, ellipse, scale=0.9] (x1) at (-4, .8) {job experience};
		\node[draw, ellipse, scale=0.9] (x2) at (0, .8) {zip code};
		\node[draw, ellipse, scale=0.9] (x3) at (4,.8) {ethnicity};
		\node[draw, ellipse, scale=0.9] (y) at (0,0) {$\hat{Y}$: predicted job performance};
		\draw[->, green] (x1) -- (y);
		\draw[->, green] (x2) -- (y);
		\draw[->, blue] (x3) -- (x2);
    \end{tikzpicture}
    \caption{Directed Acyclic Graph (DAG) illustrating the data generating process.}
    \label{fig:example-intro}
\end{figure}
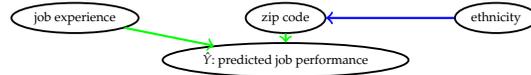
\noindent Direct importance measures such as Permutation Feature Importance (PFI) \cite{Breiman2001rf} or marginal SAGE \cite{covert_understanding_2020} consider the feature \textit{race} irrelevant, as it is not directly causal for the model's predictions. 
As a result, they fail to expose the leakage of information from \textit{ethnicity} via the proxy variable \textit{zip code}.
Measures of associative importance such as conditional SAGE \cite{covert_understanding_2020} 
would expose the association and consider \textit{ethnicity} to be relevant. 
However, even in combination with direct importance, associative importance does not provide insight into via which of the directly important features (\textit{zip code} or \textit{job experience}) the protected variable leaks into the model. Consequently we are not guided towards sparse interventions that eliminate the association from the model.\\
We propose a framework (\underline{De}composition of Global Feature Importance into \underline{D}irect and \underline{A}ssociative \underline{C}omponen\underline{t}s, DEDACT) that enables a holistic understanding of both sources of prediction-relevant information \textit{and} feature pathways which allow the respective sources of information to enter the model's predictions. With DEDACT one can either decompose the direct importance of a feature into the contributions of its sources of information, or one can decompose the associative importance of an unwanted influence into its direct feature pathway components.\\
DEDACT does not require knowledge of the underlying dependence structure or causal graph. As is the case with associative importance measures such as conditional SAGE \cite{covert_understanding_2020}, DEDACT requires sampling from conditional distributions. 
These samples can be reused for the decomposition; in this way the computational overhead can be reduced.\\
Before the decomposition will be introduced in Section \ref{sec:decomposition}, we will formally define associative and direct importance in Section \ref{sec:disentangling-direct-and-indirect} and measures for the respective direct and associative components in Section \ref{sec:iai-via-and-di-from}. The method will be validated on simulated data in Section \ref{sec:application}.
%

\subsection{Related Work and Contributions}
Previous research on local SHAP based feature effect explanations have studied direct and indirect effects. Heskes et al. \cite{heskes_causal_2020} decompose causal Shapley value functions into one direct and one indirect component. Our method is applicable to a wide range of global direct and associative feature importance methods and allows their full decomposition into feature components. Wang et al. \cite{wang_shapley_2020} offer a graph-based decomposition of direct and indirect causal effects for local causal SHAP explanations. However, their approach (1) cannot be applied to direct or associative importance methods, or (2) to feature importance methods and (3) requires knowledge of the underlying causal graph. Relative Feature Importance \cite{konig_relative_2021} quantifies the direct importance of a variable that cannot be attributed to a user specified set of variables. However, the method has not been applied to decompose the importance of a feature into its prediction-relevant sources of information. Moreover, it does not allow to decompose the associative importance of a variable. Our contribution is the first decomposition of global direct and associative importance measures into their respective associative or direct components, and therefore provides novel insight into both model and data.\\
\section{Notation}
\label{sec:background-and-notation}
The Table in Figure \ref{fig:notation} summarizes the mathematical notation and terms used throughout the article.
In order to formalize a distinction between direct and associative importance we differentiate between variables on the data level and features on the model level. We refer to the real-world concept as variable $X_j$ and to the corresponding model input as feature $\underline{X}_j$ (Figure \ref{fig:notation}, left panel). Although in the standard setting variable and feature take the same values, they are conceptually different. If we say that we intervene on a feature, it only means that we change the input to the respective feature. If we say that a feature is not causal for the prediction, the corresponding real world concept may still be causal since it could indirectly affect the prediction via downstream variables. Following \cite{janzing2020feature}, we denote perturbations of the features using the do-operator \cite{pearl2009causality}, i.e. the term $do(\underline{X}_k = \tilde{X}_k)$ denotes the replacement of the input $\underline{X}_k$ with the perturbed variable $\tilde{X}_k$. The corresponding prediction is denoted as $\hat{Y}|{do(\underline{X}_k = \tilde{X}_k)}$.\\
Some methods rely on marginalized prediction functions. To be more specific, the prediction function $f$ that is marginalized over $\underline{X}_S$ is defined as
$$ f_S(\underline{x}_{S}) = \mathbb{E}_{\underline{X}_{D \backslash S}|\underline{X}_S}[f(\underline{X}_{D \backslash S}) | \underline{X}_S = \underline{x}_{S}],$$
where the conditional expectation is taken over the feature distribution. If features are sampled independently this term reduces to the marginal expectation.\\
We indicate whether the prediction of the original model $f$ or the marginalized variant $f_S$ is meant by making use of the respective subscripts: $f$ or $f_S$. For example, $\hat{Y}_{f_S}$ refers to the marginalized prediction.\\
There will be frequent references to variables as "containing information" or to "information entering a variable". Such phrases are not meant in the information theoretic sense – despite links with concepts such as mutual information – but rather figuratively to enable a more pleasant reading experience.

\begin{figure}[H]
\hfill
\begin{tabular}{l|l} 
symbol & meaning\\
\midrule
$D$ & set of all feature indices $\{1, \dots, d\}$ \\
$X_S$, $\underline{X}_S$ & variables $S$, features $S$ \\
$Y$, $\hat{Y}$ & target, prediction \\
$f$, $f_S$ & model, marginalized model\\
&$f_S(\underline{x}_S) := \mathbb{E}_{\underline{X}_{D} |\underline{X}_S}[f(\underline{X})|\underline{X}_S=\underline{x}_S]$ \\
$\dpd, \idp$ & dependence, independence \\
$\tilde{X}_{K}^{J}$ & sample from $P(X_K|X_J)$\\
$\hat{Y}|{do(\underline{X}_K = \tilde{X}_{K}^{J})}$ & $\hat{Y}$ after intervention\\
$\mathcal{L}$ & Loss function\\
\end{tabular}
\hfill
 \begin{tikzpicture}[baseline={(current bounding box.center)}, thick, scale=1.3, every node/.style={scale=1, line width=0.25mm, black, fill=white}]
		\node[draw, circle, scale=0.7] (x1) at (0, 1) {$X_1$};
		\node[draw, circle, scale=0.7] (xj) at (-.5, 0) {$X_i$};
		\node[draw, circle, scale=0.7] (xd) at (0,-1) {$X_d$};
		\node[draw, circle, scale=0.7] (y) at (-1.25,0) {$Y$};
		\node[scale=0.7] (dots) at (-.1,.4) {$\vdots$};
		\node[scale=0.7] (dots) at (-.1,-.3) {$\vdots$};
		\draw[dashed,gray] (-1.5,-1.25) -- (.25,-1.25) -- (.25,1.25) -- (-1.5,1.25) -- cycle;
		\node[scale=0.7] (dots) at (-.625,-1.5) {data level, variables};
		\draw[-] (xj) -- (x1);
		\draw[-] (xd) -- (x1);
		\draw[-] (xd) -- (xj);
		\draw[-] (x1) -- (y);
		\draw[-] (xj) -- (y);
		\draw[-] (xd) -- (y);
		
		\node[draw, circle, scale=0.7] (ux1) at (.75, 1) {$\underline{X}_1$};
		\node[draw, circle, scale=0.7] (uxj) at (.75, 0) {$\underline{X}_i$};
		\node[draw, circle, scale=0.7] (uxd) at (.75,-1) {$\underline{X}_d$};
		\node[scale=0.7] (dots) at (.75,.6) {$\vdots$};
		\node[scale=0.7] (dots) at (.75,-.5) {$\vdots$};
		\draw[dashed,gray] (0.5,-1.25) -- (2.25,-1.25) -- (2.25,1.25) -- (0.5,1.25) -- cycle;
		\node[scale=0.7] (dots) at (1.4,-1.5) {model level, features};
		\draw[->] (x1) -- (ux1);
		\draw[->] (xj) -- (uxj);
		\draw[->] (xd) -- (uxd);
		
		\node[draw, circle, scale=0.7] (yhat) at (2, 0) {$\hat{Y}$};
		\draw[->] (ux1) -- (yhat);
		\draw[->] (uxj) -- (yhat);
		\draw[->] (uxd) -- (yhat);
\end{tikzpicture}

\hfill
\caption{Notation overview. Left: Mathematical symbols and their meaning. Right: DAG explaining the differentiation between data-level variables and model-level features.} 
\label{fig:notation}
\end{figure}
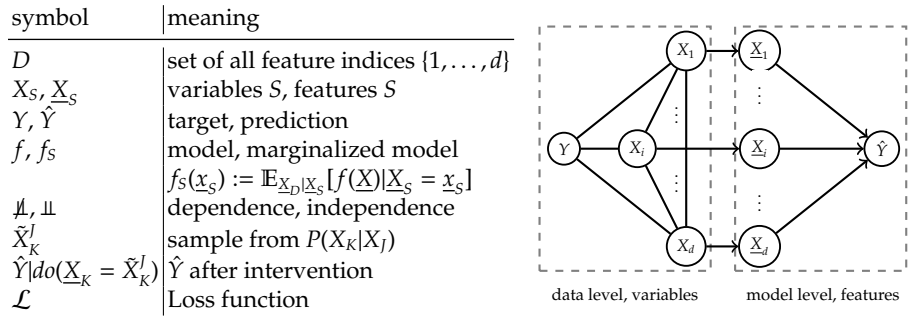

\section{Background: Direct and Associative Importance}
\label{sec:disentangling-direct-and-indirect}
Over the course of this Section, we formally introduce general definitions of direct and associative importance that generalize methods such as Permutation Feature Importance \cite{Breiman2001rf} and SAGE value functions \cite{covert_understanding_2020}. These two classes of relevance, namely direct and associative importance, are introduced in Sections \ref{subsec:dr} and \ref{subsec:ar}. As the proposed decomposition is based on these general definitions, it is applicable to a wide-range of feature importance methods.\\
Both definitions can be modified along two axes: Firstly, recent proposals \cite{Datta2017,lundberg_unified_2017,covert_understanding_2020} formulate feature relevance quantification as a cooperative game, which evaluates the relevance of features/variables with respect to different perturbation baselines.\footnote{Feature relevance attribution can be seen as a cooperative game, where the contribution of one player (feature/variable) depends on which players are already in the room \cite{lundberg_unified_2017,covert_understanding_2020}. For example, if two features contribute to the prediction performance via an interaction term, their joint contribution is attributed to the feature that is perturbed first. If the relevance of features is only considered in isolation, both features are attributed with the interaction contribution. Consequently, recent methods do not only evaluate the relevance in isolation but also over multiple different baselines \cite{covert_understanding_2020}.} In order to accommodate such considerations within our framework, the general definitions allow the specification of flexible perturbation baselines (more details below). Furthermore, the general measures accommodate methods that rely on the original prediction function $f$, such as Permutation Feature Importance (PFI) \cite{Breiman2001rf} as well as methods that marginalize the prediction function over perturbed features, as suggested by \cite{covert_understanding_2020}.\\
\begin{figure}[H]
\centering
\subfigure[DR baseline]{
\begin{tikzpicture}[thick, scale=.65, every node/.style={scale=.8, line width=0.25mm, black, fill=white}]

		\node[draw, circle, scale=0.7] (x1) at (0, 1) {$X_1$};
		\node[draw, circle, scale=0.7] (xj) at (-.5, 0) {$X_i$};
		\node[draw, circle, scale=0.7] (xd) at (0,-1) {$X_d$};
		\node[draw, circle, scale=0.7] (y) at (-1.25,0) {$Y$};
		\draw[dashed,gray] (-1.6,-1.35) -- (.35,-1.35) -- (.35,1.35) -- (-1.6,1.35) -- cycle;
		\node[scale=0.7] (dots) at (-.625,-1.6) {data};
		\draw[-] (xj) -- (x1);
		\draw[-] (xd) -- (x1);
		\draw[-] (xd) -- (xj);
		\draw[-] (x1) -- (y);
		\draw[-] (xj) -- (y);
		\draw[-] (xd) -- (y);

		\node[draw, circle, scale=0.7] (ux1) at (1, 1) {$\underline{X}_1$};
		\node[draw=orange, circle, scale=0.7] (uxj) at (1, 0) {$\underline{X}_i$};
		\node[draw=orange, circle, scale=0.7] (uxd) at (1,-1) {$\underline{X}_d$};
		\draw[dashed,gray] (0.65,-1.35) -- (2.35,-1.35) -- (2.35,1.35) -- (0.65,1.35) -- cycle;
		\node[scale=0.7] (dots) at (1.5,-1.6) {model};
		\draw[->] (x1) -- (ux1);
		
		\node[draw, circle, scale=0.7] (yhat) at (2, 0) {$\hat{Y}$};
		\draw[->] (ux1) -- (yhat);
		\draw[->] (uxj) -- (yhat);
		\draw[->] (uxd) -- (yhat);
		
	
\end{tikzpicture}
}
\subfigure[DR restored]{
\begin{tikzpicture}[thick, scale=.65, every node/.style={scale=.8, line width=0.25mm, black, fill=white}]

		\node[draw, circle, scale=0.7] (x1) at (0, 1) {$X_1$};
		\node[draw, circle, scale=0.7] (xj) at (-.5, 0) {$X_i$};
		\node[draw, circle, scale=0.7] (xd) at (0,-1) {$X_d$};
		\node[draw, circle, scale=0.7] (y) at (-1.25,0) {$Y$};
		\draw[dashed,gray] (-1.6,-1.35) -- (.35,-1.35) -- (.35,1.35) -- (-1.6,1.35) -- cycle;
		\node[scale=0.7] (dots) at (-.625,-1.6) {data};
		\draw[-,blue] (xj) -- (x1);
		\draw[-] (xd) -- (x1);
		\draw[-,blue] (xd) -- (xj);
		\draw[-] (x1) -- (y);
		\draw[-] (xj) -- (y);
		\draw[-] (xd) -- (y);

		\node[draw, circle, scale=0.7] (ux1) at (1, 1) {$\underline{X}_1$};
		\node[draw=red, circle, scale=0.7] (uxj) at (1, 0) {$\underline{X}_i$};
		\node[draw=orange, circle, scale=0.7] (uxd) at (1,-1) {$\underline{X}_d$};
		\draw[dashed,gray] (0.65,-1.35) -- (2.35,-1.35) -- (2.35,1.35) -- (0.65,1.35) -- cycle;
		\node[scale=0.7] (dots) at (1.5,-1.6) {model};
		\draw[->] (x1) -- (ux1);
		\draw[->,red] (xj) -- (uxj);
		
		\node[draw, circle, scale=0.7] (yhat) at (2, 0) {$\hat{Y}$};
		\draw[->] (ux1) -- (yhat);
		\draw[->,green] (uxj) -- (yhat);
		\draw[->] (uxd) -- (yhat);
		
	
\end{tikzpicture}
}
\subfigure[AR baseline]{
\begin{tikzpicture}[thick, scale=.65, every node/.style={scale=.8, line width=0.25mm, black, fill=white}]

		\node[draw, circle, scale=0.7] (x1) at (0, 1) {$X_1$};
		\node[draw, circle, scale=0.7] (xj) at (-.5, 0) {$X_i$};
		\node[draw, circle, scale=0.7] (xd) at (0,-1) {$X_d$};
		\node[draw, circle, scale=0.7] (y) at (-1.25,0) {$Y$};
		\draw[dashed,gray] (-1.6,-1.35) -- (.35,-1.35) -- (.35,1.35) -- (-1.6,1.35) -- cycle;
		\node[scale=0.7] (dots) at (-.625,-1.6) {data};
		\draw[-] (xj) -- (x1);
		\draw[-] (xd) -- (x1);
		\draw[-] (xd) -- (xj);
		\draw[-] (x1) -- (y);
		\draw[-] (xj) -- (y);
		\draw[-] (xd) -- (y);

		\node[draw, circle, scale=0.7] (ux1) at (1, 1) {$\underline{X}_1$};
		\node[draw=orange, circle, scale=0.7] (uxj) at (1, 0) {$\underline{X}_i$};
		\node[draw=orange, circle, scale=0.7] (uxd) at (1,-1) {$\underline{X}_d$};
		\draw[dashed,gray] (0.65,-1.35) -- (2.35,-1.35) -- (2.35,1.35) -- (0.65,1.35) -- cycle;
		\node[scale=0.7] (dots) at (1.5,-1.6) {model};
		\draw[->] (x1) -- (ux1);
		\draw[->,orange] (x1) -- (uxj);
		\draw[->,orange] (x1) -- (uxd);
		
		\node[draw, circle, scale=0.7] (yhat) at (2, 0) {$\hat{Y}$};
		\draw[->] (ux1) -- (yhat);
		\draw[->] (uxj) -- (yhat);
		\draw[->] (uxd) -- (yhat);
		
\end{tikzpicture}
}
\subfigure[AR restored]{
\begin{tikzpicture}[thick, scale=.65, every node/.style={scale=.8, line width=0.25mm, black, fill=white}]

		\node[draw, circle, scale=0.7] (x1) at (0, 1) {$X_1$};
		\node[draw, circle, scale=0.7] (xj) at (-.5, 0) {$X_i$};
		\node[draw, circle, scale=0.7] (xd) at (0,-1) {$X_d$};
		\node[draw, circle, scale=0.7] (y) at (-1.25,0) {$Y$};
		\draw[dashed,gray] (-1.6,-1.35) -- (.35,-1.35) -- (.35,1.35) -- (-1.6,1.35) -- cycle;
		\node[scale=0.7] (dots) at (-.625,-1.6) {data};
		\draw[-] (xj) -- (x1);
		\draw[-] (xd) -- (x1);
		\draw[-] (xd) -- (xj);
		\draw[-] (x1) -- (y);
		\draw[-] (xj) -- (y);
		\draw[-] (xd) -- (y);

		\node[draw, circle, scale=0.7] (ux1) at (1, 1) {$\underline{X}_1$};
		\node[draw=red, circle, scale=0.7] (uxj) at (1, 0) {$\underline{X}_i$};
		\node[draw=red, circle, scale=0.7] (uxd) at (1,-1) {$\underline{X}_d$};
		\draw[dashed,gray] (0.65,-1.35) -- (2.35,-1.35) -- (2.35,1.35) -- (0.65,1.35) -- cycle;
		\node[scale=0.7] (dots) at (1.5,-1.6) {model};
		\draw[->,red] (xj) -- (uxd);
		\draw[->,red] (xj) -- (uxj);
		\draw[->] (x1) -- (ux1);
		\draw[->,orange] (x1) -- (uxj);
		\draw[->,orange] (x1) -- (uxd);
		
		\node[draw, circle, scale=0.7] (yhat) at (2, 0) {$\hat{Y}$};
		\draw[->,black] (ux1) -- (yhat);
		\draw[->,green] (uxj) -- (yhat);
		\draw[->,blue] (uxd) -- (yhat);
		
\end{tikzpicture}
}
\caption{The DAGs (a) - (d) illustrate the difference between the Direct and Associative Importance of a variable of interest $K=J=\{i\}$ and baseline sets $B = C = \{1\}$. (a) While the baseline $X_B$ is restored ($\underline{X}_B = X_B$), all remaining model level features $\underline{X}$ are perturbed (orange), breaking all links to the prediction relevant information in the data level variables $X$. (b) For DI, only the feature of interest is restored (red) ($\underline{X}_K = X_K$), all other features are unaffected. Therefore, only the relevance via the direct path (green) is quantified. As $X_i$ is dependent with $X_1$ and $X_d$ (blue), the feature $\underline{X}_i$ now also contains information about the dependent covariates. (c) For AI, in the baseline setting the features are reconstructed from context variables $X_C$ (meaning that $\underline{X}_C$ is fully restored). (d) In addition to the information from $X_C$, the information from $X_J$ is restored all features $\underline{X}$ (affected features in red). As total associative importance intervenes on all features that are dependent with $X_J$ given $X_C$, influence enters not only via direct (green), but also via indirect pathways (blue).}
\label{fig:dag-intuition-dr-ar} 
\end{figure}
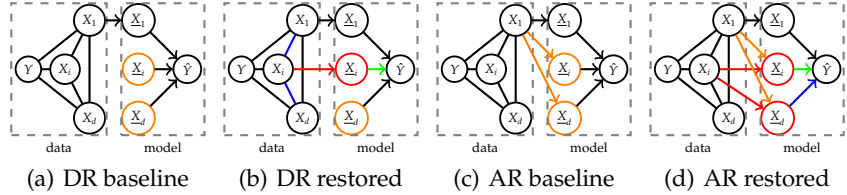
\subsection{Direct Importance}
\label{subsec:dr}

Direct importance measures assess the relevance of specific features for the model's performance. They are based on a simple idea: If a feature is not important, removing all its prediction-relevant information should have no or only little effect on the model's performance.
\\
As such, direct importance measures compare the model performance before and after perturbing the features of interest. The perturbed variable $\tilde{X}_K^\emptyset$ preserves the marginal distribution of the original variable, but is sampled independently from $X$ and $Y$ (as indicated by the superscript $\emptyset$). Hence, the perturbed variable only contains noise and no prediction-relevant information.\\
Our notation starts off with the full perturbation of all features. After that, we quantify the improvement in performance when restoring a set of features. That is, the direct importance $DI(\underline{X}_K)$ quantifies the improvement in performance when features $\underline{X}_K$ are restored in isolation. The direct importance $DI(\underline{X}_K|\underline{X}_B)$ over baseline $B$ quantifies the improvement in performance, when $\underline{X}_K$ is restored in a context where features $\underline{X}_B$ have been restored already. Furthermore, we indicate whether the original model $f$ or the marginalized prediction functions $f_{B}, f_{B \cup K}$ are used by $DI(\underline{X}_K|\underline{X}_B; f, f)$ and $DI(\underline{X}_K|\underline{X}_B; f_{B}, f_{B \cup K})$. For example, Permutation Feature Importance is a special case of a direct importance measure, i.e. $DI(\underline{X}_K|\underline{X}_{D \backslash K}; f, f)$. Direct importance is formally introduced in Definition \ref{def:dr} and illustrated in Figure \ref{fig:dag-intuition-dr-ar} (a) - (b).\\
\begin{definition}[Direct Importance (DI)]
Given two disjoint feature index sets $B$ and $K$, direct importance of features $\underline{X}_K$ over baseline features $\underline{X}_B$ is defined as:
\begin{align*}
    DI(\underline{X}_K|\underline{X}_B;g,g') := &\mathbb{E}[\loss(\hat{Y}_{g}|{do(\underline{X}_{D \backslash B }=\tilde{X}_{D \backslash B}^\emptyset)};Y)]\\ 
    - &\mathbb{E}[\loss(\hat{Y}_{g'}|{do(\underline{X}_{D \backslash (B \cup K)} = \tilde{X}_{D \backslash (B \cup K)}^\emptyset)};Y)]
\end{align*}
where $(g, g') \in \{(f, f),(f_{B}, f_{(B \cup K)})\}$. In addition, the perturbed joint $\tilde{X}^{\emptyset}$ preserves the covariate joint distribution, but is independent of the original variables and the prediction target. I.e.,  $P(\tilde{X}^{\emptyset}) = P(X) \text{ and }\tilde{X}^{\emptyset} \idp (X, Y)$.
\label{def:dr}
\end{definition}
As all features except the features of interest are left unchanged, direct importance can only enter the model via the features of interest themselves. Therefore, directly important features must be causal for the prediction (Proposition \ref{proposition:tdr-direct-sensitivity}).\\
\begin{proposition}[Direct Sensitivity of DI]
Nonzero Direct Importance $DI(\underline{X}_K|\underline{X}_B)$ implies that feature $\underline{X}_K$ is causal for the model (over $P(X_B)P(X_K)P(X_{D \backslash (B \cup K)})$).
\label{proposition:tdr-direct-sensitivity}
\end{proposition}
However, as illustrated by the motivational example in the introduction to this paper, direct importance fails to expose leakage of information from variables that are not directly used by the model.
\subsection{Associative Importance}
\label{subsec:ar}
In contrast to direct importance, associative importance measures are not concerned with the mechanisms that the model uses. They rather aim at measuring the component of the predictive performance that can be explained with information contained in the variables of interest $X_J$.\\
Therefore, associative importance measures compare the model's performance under fully perturbed features with the performance that results from reconstruction of the features by the use of information from the variables of interest $X_J$. The reconstructed features preserve the conditional distribution of their corresponding variables with the variable of interest, i.e. $P(\underline{X}|X_J) = P(X|X_J)$. Consequently, the features corresponding to the variables of interest $\underline{X}_J$ are fully reconstructed, features that correspond to dependent variables are partly reconstructed, and features that correspond to variables that are independent from $X_J$ remain unchanged.\\
As for direct importance, it is of relevance whether the variable is introduced in isolation, or whether it is added to a set of variables for which the information has been restored already. We denote the associative importance of variables $X_J$ that are introduced in isolation as $AI(X_J)$, and the associative importance of variables $X_J$ that are introduced to already conditioned upon variables $X_C$ as $AI(X_J|X_C)$. Associative importance is formally introduced in Definition \ref{def:ar} and illustrated in Figure \ref{fig:dag-intuition-dr-ar} (c) - (d).
\begin{definition}[Associative Importance (AI)]
We define the associative importance of a variable set $J$ given context variables $C$ as:
\begin{align*}
 AI(X_J|X_C; g, g') := &\mathbb{E}[\loss(Y; \hat{Y}_g|{do(\underline{X}_{D\backslash C} = \tilde{X}_{D\backslash C}^{C})})]\\
- &\mathbb{E}[\loss(Y;\hat{Y}_{g'}|{do(\underline{X}_{D\backslash (C \cup J)} = \tilde{X}_{D \backslash (C \cup J)}^{C \cup J})})]
\end{align*}
where $(g,g') \in \{(f,f), (f_{C},f_{C \cup J})\}$.
The conditionally perturbed joint $\tilde{X}^{S}$ is required to preserve the covariate joint distribution as well as the conditional distribution with features $X_S$, i.e. $P(\tilde{X}^{S}| X_S) = P(X|X_S)$. Also, it is conditionally independent of the remaining covariates and prediction target  ($\tilde{X}^{S} \idp (X_{D \backslash S}, Y) | X_S$).
\label{def:ar}
\end{definition}
Associative importance provides insight into the dependence structure in the underlying dataset. Nonzero associative importance indicates that $X_J$ contains prediction-relevant information that is not included in the context variables $X_C$.
\begin{proposition} [Associative Sensitivity of AI]
Nonzero associative importance $AI(X_J|X_C;f,f)$ based on the non-marginalized function $f$ implies that the target $Y$ is conditionally dependent with $X_J$ given $X_C$, i.e. $Y \dpd X_J | X_C$. If the marginalized functions $f_C$ and $f_{C \cup J}$ are used, nonzero associative importance $AI(X_J|X_C;f_C, f_{C \cup J})$ implies conditional dependence of $Y$ with $X_J$ given $X_C$, i.e. $Y \dpd X_J | X_C$ if (1) cross entropy or mean squared error is used as loss and (2) $f$ is the respective loss-optimal predictor.
\label{proposition:tar-associative-sensitivity}
\end{proposition}
Yet, for the computation of associative importance not only the distribution of the features $\underline{X}_J$ that correspond to the variables of interest $X_J$ but also all the features corresponding to dependent variables of $X_J$ is affected. Associative importance therefore enters the model directly and/or indirectly. As a result, it fails to provide insight into which features the model uses for its prediction. An illustrative example for the lacking causal interpretation of associative importance is given in Section \ref{sec:introduction}.\\
With the right choice of parameters, the definitions of associative importance and direct importance yield well-established IML methods. An overview is given in Table \ref{table:influence-family-members}. In order to estimate the measures based on real data, the expected loss (risk) can be replaced by the empirical risk.\\
\begin{table}[H]
\centering
\begin{tabular}{l|l|l|l}
\textbf{type} & \textbf{$(g,g')$} & \textbf{sets} & \textbf{method}\\
\midrule
\textbf{DI}  & $(f, f)$ & $K = \{k\}$ $B = D \backslash K$, & Permutation Feature Importance of $k$ \cite{Breiman2001rf}\\

 & $(f_{\emptyset}, f_{S})$ & $K = S$, $B = \emptyset$ & marginal SAGE $v(S)$ \cite{covert_understanding_2020} \\

\midrule
\textbf{AI} & $(f_{C}, f_{C \cup J})$ & $J = S$, $C = \emptyset$ & conditional SAGE $v(S)$ \cite{covert_understanding_2020}\\

 & $(f, f)$ & $J=\{j\}$, $C = D \backslash J$ & Conditional Feature Importance of $j$ \cite{Strobl2008}\\
\end{tabular}
\caption{Parameter choices for direct importance (DI) and associative importance (AI) that yield well-established IML methods.}
\label{table:influence-family-members}
\end{table}
\section{Components of Associative and Direct Importance}
\label{sec:iai-via-and-di-from}
In Section \ref{sec:disentangling-direct-and-indirect} classes of direct and associative importance measures have been formally introduced. We have shown that both measures have strengths and weaknesses: While direct importance measures allow the exposure of those features that causally influence the prediction performance, they fail to uncover leakage of information from variables that the model does not directly use for its prediction. In contrast to that, associative importance measures can uncover leakage of information from variables of interest, irrespective of whether the corresponding feature is causal for the model's prediction and performance. However, as opposed to direct importance, associative importance measures do not incline a causal interpretation: If the associative importance for a variable is nonzero, we cannot infer that the corresponding feature is causal for the prediction. Even in combination, the methods fail to provide insight into both the sources of prediction-relevant information and the feature pathways, which allow the information to enter the model.\\
In order to mitigate this inherent trade-off, we propose two measures. Firstly, \textit{DI from} that quantifies the component of the direct importance of a feature of interest $\underline{X}_K$ that can be explained with information from a set of variables $X_J$. Using \textit{DI from}, we get insight into the sources of information that enter the model via a feature $\underline{X}_K$. Secondly, \textit{AI via} that quantifies the component of the associative importance of a variable $X_J$ that enters the model via a specific set of features $\underline{X}_K$. Using \textit{AI via}, we learn which feature pathways allow the information of variables $X_J$ to enter the model.\\
\begin{figure}[H]
\centering
\subfigure["DI from" before]{
\begin{tikzpicture}[thick, scale=.65, every node/.style={scale=.8, line width=0.25mm, black, fill=white}]
		\node[draw, circle, scale=0.7] (x1) at (0, 1) {$X_1$};
		\node[draw, circle, scale=0.7] (xj) at (-.5, 0) {$X_i$};
		\node[draw, circle, scale=0.7] (xd) at (0,-1) {$X_d$};
		\node[draw, circle, scale=0.7] (y) at (-1.25,0) {$Y$};
		\draw[dashed,gray] (-1.6,-1.35) -- (.35,-1.35) -- (.35,1.35) -- (-1.6,1.35) -- cycle;
		\node[scale=0.7] (dots) at (-.625,-1.6) {data};
		\draw[-] (xj) -- (x1);
		\draw[-] (xd) -- (x1);
		\draw[-] (xd) -- (xj);
		\draw[-] (x1) -- (y);
		\draw[-] (xj) -- (y);
		\draw[-] (xd) -- (y);

		\node[draw, circle, scale=0.7] (ux1) at (1, 1) {$\underline{X}_1$};
		\node[draw=orange, circle, scale=0.7] (uxj) at (1, 0) {$\underline{X}_i$};
		\node[draw=orange, circle, scale=0.7] (uxd) at (1,-1) {$\underline{X}_d$};
		\draw[dashed,gray] (0.65,-1.35) -- (2.35,-1.35) -- (2.35,1.35) -- (0.65,1.35) -- cycle;
		\node[scale=0.7] (dots) at (1.5,-1.6) {model};
		\draw[->] (x1) -- (ux1);

		\node[draw, circle, scale=0.7] (yhat) at (2, 0) {$\hat{Y}$};
		\draw[->] (ux1) -- (yhat);
		\draw[->] (uxj) -- (yhat);
		\draw[->] (uxd) -- (yhat);
	
\end{tikzpicture}
}
\subfigure["DI from" after]{
\begin{tikzpicture}[thick, scale=.65, every node/.style={scale=.8, line width=0.25mm, black, fill=white}]
		\node[draw, circle, scale=0.7] (x1) at (0, 1) {$X_1$};
		\node[draw, circle, scale=0.7] (xj) at (-.5, 0) {$X_i$};
		\node[draw, circle, scale=0.7] (xd) at (0,-1) {$X_d$};
		\node[draw, circle, scale=0.7] (y) at (-1.25,0) {$Y$};
		\draw[dashed,gray] (-1.6,-1.35) -- (.35,-1.35) -- (.35,1.35) -- (-1.6,1.35) -- cycle;
		\node[scale=0.7] (dots) at (-.625,-1.6) {data};
		\draw[-] (xj) -- (x1);
		\draw[-] (xd) -- (x1);
		\draw[-] (xd) -- (xj);
		\draw[-] (x1) -- (y);
		\draw[-] (xj) -- (y);
		\draw[-] (xd) -- (y);
		
		\node[draw, circle, scale=0.7] (ux1) at (1, 1) {$\underline{X}_1$};
		\node[draw=red, circle, scale=0.7] (uxj) at (1, 0) {$\underline{X}_i$};
		\node[draw=orange, circle, scale=0.7] (uxd) at (1,-1) {$\underline{X}_d$};
		\draw[dashed,gray] (0.65,-1.35) -- (2.35,-1.35) -- (2.35,1.35) -- (0.65,1.35) -- cycle;
		\node[scale=0.7] (dots) at (1.5,-1.6) {model};
		\draw[->] (x1) -- (ux1);
		\draw[->,red] (xd) -- (uxj);

		\node[draw, circle, scale=0.7] (yhat) at (2, 0) {$\hat{Y}$};
		\draw[->] (ux1) -- (yhat);
		\draw[->,blue] (uxj) -- (yhat);
		\draw[->] (uxd) -- (yhat);
\end{tikzpicture}
}
\subfigure["AI via" before]{
\begin{tikzpicture}[thick, scale=.65, every node/.style={scale=.8, line width=0.25mm, black, fill=white}]
		\node[draw, circle, scale=0.7] (x1) at (0, 1) {$X_1$};
		\node[draw, circle, scale=0.7] (xj) at (-.5, 0) {$X_i$};
		\node[draw, circle, scale=0.7] (xd) at (0,-1) {$X_d$};
		\node[draw, circle, scale=0.7] (y) at (-1.25,0) {$Y$};
		\draw[dashed,gray] (-1.6,-1.35) -- (.35,-1.35) -- (.35,1.35) -- (-1.6,1.35) -- cycle;
		\node[scale=0.7] (dots) at (-.625,-1.6) {data};
		\draw[-] (xj) -- (x1);
		\draw[-] (xd) -- (x1);
		\draw[-] (xd) -- (xj);
		\draw[-] (x1) -- (y);
		\draw[-] (xj) -- (y);
		\draw[-] (xd) -- (y);
		
		\node[draw=orange, circle, scale=0.7] (ux1) at (1, 1) {$\underline{X}_1$};
		\node[draw=orange, circle, scale=0.7] (uxj) at (1, 0) {$\underline{X}_i$};
		\node[draw=orange, circle, scale=0.7] (uxd) at (1,-1) {$\underline{X}_d$};
		\draw[dashed,gray] (0.65,-1.35) -- (2.35,-1.35) -- (2.35,1.35) -- (0.65,1.35) -- cycle;
		\node[scale=0.7] (dots) at (1.5,-1.6) {model};
		\draw[->,orange] (x1) -- (ux1);
		\draw[->,orange] (x1) -- (uxj);
		\draw[->,orange] (x1) -- (uxd);
		
		\node[draw, circle, scale=0.7] (yhat) at (2, 0) {$\hat{Y}$};
		\draw[->] (ux1) -- (yhat);
		\draw[->] (uxj) -- (yhat);
		\draw[->] (uxd) -- (yhat);
\end{tikzpicture}
}
\subfigure["AI via" after]{
\begin{tikzpicture}[thick, scale=.65, every node/.style={scale=.8, line width=0.25mm, black, fill=white}]

		\node[draw, circle, scale=0.7] (x1) at (0, 1) {$X_1$};
		\node[draw, circle, scale=0.7] (xj) at (-.5, 0) {$X_i$};
		\node[draw, circle, scale=0.7] (xd) at (0,-1) {$X_d$};
		\node[draw, circle, scale=0.7] (y) at (-1.25,0) {$Y$};
		\draw[dashed,gray] (-1.6,-1.35) -- (.35,-1.35) -- (.35,1.35) -- (-1.6,1.35) -- cycle;
		\node[scale=0.7] (dots) at (-.625,-1.6) {data};
		\draw[-] (xj) -- (x1);
		\draw[-] (xd) -- (x1);
		\draw[-] (xd) -- (xj);
		\draw[-] (x1) -- (y);
		\draw[-] (xj) -- (y);
		\draw[-] (xd) -- (y);

		\node[draw=orange, circle, scale=0.7] (ux1) at (1, 1) {$\underline{X}_1$};
		\node[draw=orange, circle, scale=0.7] (uxj) at (1, 0) {$\underline{X}_i$};
		\node[draw=red, circle, scale=0.7] (uxd) at (1,-1) {$\underline{X}_d$};
		\draw[dashed,gray] (0.65,-1.35) -- (2.35,-1.35) -- (2.35,1.35) -- (0.65,1.35) -- cycle;
		\node[scale=0.7] (dots) at (1.5,-1.6) {model};

		\draw[->,red] (xj) -- (uxd);
		\draw[->,orange] (x1) -- (ux1);
		\draw[->,orange] (x1) -- (uxj);
		\draw[->,orange] (x1) -- (uxd);
		
		\node[draw, circle, scale=0.7] (yhat) at (2, 0) {$\hat{Y}$};
		\draw[->,black] (ux1) -- (yhat);
		\draw[->,black] (uxj) -- (yhat);
		\draw[->,blue] (uxd) -- (yhat);
	
\end{tikzpicture}
}
\caption{The DAGs (a) - (d) illustrate $AI(X_i|X_1 \rightarrow \underline{X}_d)$ and $DI(\underline{X}_i|\underline{X}_1 \leftarrow X_d)$. (a-b) For $DI(\underline{X}_i|\underline{X}_1 \leftarrow X_d)$ we only partially reconstruct feature $\underline{X}_i$ using variable $X_d$ (c-d) In order to compute $AI(X_i|X_1 \rightarrow \underline{X}_d)$, we only reconstruct the information from the variable $X_i$ in the feature of interest $\underline{X}_d$.}
\label{fig:dag-intuition-tai-via-di-from} 
\end{figure}
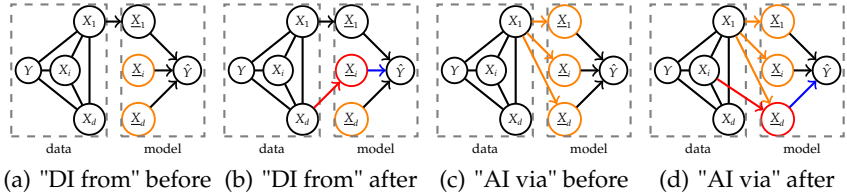
\subsection{Components of Direct Importance}
When computing the direct importance of a feature, we compare the prediction when $\underline{X}_K$ is fully perturbed with the prediction when $\underline{X}_K$ is fully reconstructed. Now, we would like to understand whether $X_J$ contributes to the direct importance of $\underline{X}_K$. Therefore, instead of fully reconstructing $\underline{X}_K$ we only partially reconstruct the feature using information from $X_J$. As such, we quantify the degree to which the direct importance of $\underline{X}_K$ can be explained with $X_J$. An illustration of this procedure is given in Figure \ref{fig:dag-intuition-tai-via-di-from} (a)-(b).\\
\begin{definition}[DI from]
For the direct importance of features $K$ given baseline $B$, we define the associative component that can be attributed to variables $J$ as 
\begin{align*}
	DI(\underline{X}_K|\underline{X}_B \leftarrow X_J;g, g')
	:= &\mathbb{E}[\loss(Y; \hat{Y}_g|{do(\underline{X}_{D \backslash B} = \tilde{X}_{D \backslash B}^{\emptyset})})] \\
	- &\mathbb{E}[\loss(Y; \hat{Y}_{g'}|{do(\underline{X}_{D \backslash B} = \tilde{X}_{D \backslash B}^{\emptyset},  \underline{X}_{K} = \tilde{X}_{K}^{J})})]
\end{align*}
where $(g, g') \in \{(f, f),(f_{B}, f_{(B \cup K)})\}$. In addition, $\tilde{X}^\emptyset$ and $\tilde{X}^J$ satisfy the following requirements: $P(\tilde{X}^\emptyset) = P(X)$ and $\tilde{X}^\emptyset \idp X$ as well as $P(\tilde{X}^J|X_J) = P(X|X_J)$ and $\tilde{X}^J \idp X | X_J$.
\label{def:tdi-from}
\end{definition}
Intuitively, if the variables $X_J$ do not contain information, the respective features $\underline{X}_K$ cannot be reconstructed at all. Consequently, the direct importance $DI(\underline{X}_K|\underline{X}_B)$ cannot be attributed to $X_J$. In contrast, if $X_J$ were to be perfectly correlated with $\underline{X}_K$, the feature is fully reconstructed and the full relevance can be attributed to $X_J$.\\
In scenarios where all features except for the features of interest are reconstructed in the baseline $B = D \backslash K$, \textit{DI from} is complementary to Relative Feature Importance \cite{konig_relative_2021}, which quantifies the component of Permutation Feature Importance that cannot be attributed to a user-defined set of variables $X_G$.\\
    
\subsection{Components of Associative Importance}

When computing the associative importance $AI(X_J|X_C)$, we compare the prediction when \textit{all} features are reconstructed using $X_C$ with the prediction when \textit{all} features are reconstructed using $X_C$ and $X_J$. In order to gain insight into the direct feature contributions to the associative importance, we want to quantify how much specific feature pathways $\underline{X}_K$ contribute to the overall score. Therefore, instead of reconstructing the information from $X_J$ in all features, we only update the features of interest $\underline{X}_K$. In other words, we allow information from $X_J$ over $X_C$ to enter the model via $\underline{X}_K$ while blocking all other feature paths. An illustration is provided in Figure \ref{fig:dag-intuition-tai-via-di-from} (c)-(d).\\
\begin{definition}[AI via]
The Associative Importance of variable set $X_J$ given context variables $X_C$ via features $\underline{X}_K$ is defined as
\begin{align*}
	AI(X_J|X_C \rightarrow \underline{X}_K;g,g') 
	= &\mathbb{E}[\loss(Y; \hat{Y}_{g}|{do(\underline{X}_D = \tilde{X}_D^{C})})] \\
	- &\mathbb{E}[\loss(Y; \hat{Y}_{g'}|{do(\underline{X}_{K} = \tilde{X}_{K}^{C \cup J}, \underline{X}_{D \backslash K} = \tilde{X}_{D \backslash K}^{C})})]\\
	= &AI(X_J|X_C;g,g') | {do(\underline{X}_{D \backslash K} = \tilde{X}_{D \backslash K}^{C})}
\end{align*}
for $(g,g') \in \{(f,f), (f_{C},f_{C \cup J})\}$. Furthermore, $\tilde{X}^C$ and $\tilde{X}^{C \cup J}$ satisfy the following requirements: $P(\tilde{X}^C | X_C) = P(X | X_C)$ and  $\tilde{X}^C \idp X | X_C$ as well as $P(\tilde{X}^{C \cup J} | X_{C \cup J}) = P(X | X_{C \cup J})$ and $\tilde{X}^{C \cup J} \idp X | X_{C \cup J}$.
\label{def:tai-via}
\end{definition}
Intuitively, \textit{AI via} can be thought of as the direct importance of a feature of interest, except that instead of fully perturbing the features only surplus information w.r.t to $X_C$ is removed, and that features are only reconstructed with respect to $X_C$ and $X_J$ instead of fully reconstructing them.\\
\\
It should be noted, that both measures are based on model evaluations on mixed joint distributions of the form $\tilde{X}_{K}^{C \cup J}, \tilde{X}_{D \backslash K}^{C}$. As the variables are being defined independently, the components are only linked via the conditioning variable $C$. Consequently, if $X_K$ and $X_{D \backslash K}$ are dependent conditional on $X_C$, the mixed joint distribution does not preserve the original covariate joint. Such extrapolation may cause misleading interpretations \cite{hooker_please_2019,molnar2020pitfalls}. We conjecture that coupling the variables, e.g. with conditional-rank preserving independent representation learning, can avoid unnecessary extrapolation \cite{lum_statistical_2016,johndrow_algorithm_2019,Zhao2020Conditional}. However, this investigation goes beyond the scope of the article.

\section{Decomposing Associative and Direct Importance with DEDACT}
\label{sec:decomposition}

In the previous Section, we have defined two measures: \textit{DI from} that quantifies the component of a direct importance that can be attributed to a set of variables $X_J$. And \textit{AI via} that quantifies the component of the associative importance that enters the model via a set of features $\underline{X}_K$.\\
We can leverage these measures in various ways to yield a decomposition of associative and direct importance. The choice of decomposition depends on the prerequisites we place on the method: For scenarios where computational efficiency is prioritized, we introduce a fast decomposition in Section \ref{subsec:fast-decomposition}. For situations where computational resources are available and a fair attribution is prioritized, we propose an additive Shapley value based decomposition in Section \ref{subsec:shapley-decomposition}.
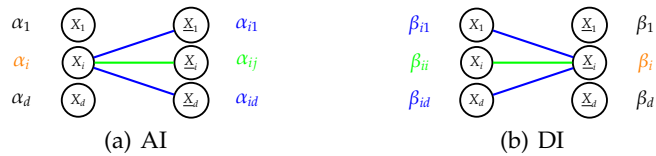
\begin{figure}[H]
\centering
    \hfill
\subfigure[AI]{
    \begin{tikzpicture}[thick, scale=.5, every node/.style={scale=.85, line width=0.25mm, black, fill=white}]

		\node[draw, circle, scale=0.7] (x1) at (-1.5, 1) {$X_1$};
		\node[draw, circle, scale=0.7] (xj) at (-1.5, 0) {$X_i$};
		\node[draw, circle, scale=0.7] (xd) at (-1.5,-1) {$X_d$};
		
		\node[scale=1] (dots) at (-3,1) {$\alpha_1$};
		\node[scale=1, text=orange] (dots) at (-3,0) {$\alpha_i$};
		\node[scale=1] (dots) at (-3,-1) {$\alpha_d$};
		
		\node[draw, circle, scale=0.7] (ux1) at (1.5, 1) {$\underline{X}_1$};
		\node[draw, circle, scale=0.7] (uxj) at (1.5, 0) {$\underline{X}_i$};
		\node[draw, circle, scale=0.7] (uxd) at (1.5,-1) {$\underline{X}_d$};
		
		\node[scale=1, text=blue] (dots) at (3,1) {$\alpha_{i1}$};
		\node[scale=1, text=green] (dots) at (3,0) {$\alpha_{ij}$};
		\node[scale=1, text=blue] (dots) at (3,-1) {$\alpha_{id}$};		

		\draw[-, blue] (xj) -- (ux1);
		\draw[-, blue] (xj) -- (uxd);
		\draw[-, green] (xj) -- (uxj);
		
    \end{tikzpicture}
}
\hfill
\subfigure[DI]{
    \begin{tikzpicture}[thick, scale=.5, every node/.style={scale=.85, line width=0.25mm, black, fill=white}]

		\node[draw, circle, scale=0.7] (x1) at (-1.5, 1) {$X_1$};
		\node[draw, circle, scale=0.7] (xj) at (-1.5, 0) {$X_i$};
		\node[draw, circle, scale=0.7] (xd) at (-1.5,-1) {$X_d$};
		
		\node[scale=1, text=blue] (dots) at (-3,1) {$\beta_{i1}$};
		\node[scale=1, text=green] (dots) at (-3,0) {$\beta_{ii}$};
		\node[scale=1, text=blue] (dots) at (-3,-1) {$\beta_{id}$};
		
		\node[draw, circle, scale=0.7] (ux1) at (1.5, 1) {$\underline{X}_1$};
		\node[draw, circle, scale=0.7] (uxj) at (1.5, 0) {$\underline{X}_i$};
		\node[draw, circle, scale=0.7] (uxd) at (1.5,-1) {$\underline{X}_d$};
		
		\node[scale=1] (dots) at (3,1) {$\beta_{1}$};
		\node[scale=1, text=orange] (dots) at (3,0) {$\beta_{i}$};
		\node[scale=1] (dots) at (3,-1) {$\beta_{d}$};

		\draw[-, blue] (x1) -- (uxj);
		\draw[-, green] (xj) -- (uxj);
		\draw[-, blue] (xd) -- (uxj);
		
    \end{tikzpicture}
}
\hfill
    \caption{Notation for decompositions of direct and associative importance. (a) The associative importance of a variable (denoted as $\alpha_j$) is decomposed into its feature contributions (denoted as $\alpha_{jk}$). (b) The direct importance of a feature (denoted as $\beta_k$) is decomposed into its variable contributions (denoted as $\beta_{kj}$). The feature/variable of interest is highlighted in orange, the respective direct and indirect components are highlighted in green and blue.}
    \label{fig:my_label}
\end{figure}
\subsection{DEDACT: Fast Decomposition}
\label{subsec:fast-decomposition}

For the purpose of computational efficiency, we restrict the fast decomposition to require the evaluation of only one further \textit{AI via}/\textit{DI from} configuration per evaluation of $AI$/$DI$ and component. Furthermore, we prioritize sensitivity, referred to the exposure of sources of information and feature pathways, over additivity. One direct and one associative importance measure will be considered as examples.\\
\\
\textbf{Permutation Feature Importance:} In our notation, Permutation Feature Importance \cite{Breiman2001rf} is written as:
$$\beta_k := PFI_k = DI(\underline{X}_k | \underline{X}_{D \backslash k}; f, f)$$
Using a single evaluation of \textit{DI from}, we can quantify the component of the feature importance that can be explained by a specific variable $X_j$.
$$\beta_{kj} := DI(\underline{X}_k | \underline{X}_{D \backslash k} \leftarrow X_j; f, f)$$
The component $\beta_{kj}$ represents the total contribution that the variable $X_j$ is able to provide in isolation. Nevertheless, the contribution that $X_j$ can only deliver in cooperation with a further feature $X_l$ is not quantified. Further, if variables $X_j$ and $X_l$ are dependent, then both are fully attributed with the shared contribution. Consequently, the attributions are not guaranteed to add up to the overall importance $\beta_k$.\\
In order to achieve an \textit{additive} decomposition with just one evaluation per feature and component, we would need to define an order of the variables $\pi$. This estimate would then be
$$\beta_{kj}^{\pi} = DI(\underline{X}_K | \underline{X}_{D \backslash k} \leftarrow X_{S \cup j}) - DI(\underline{X}_K | \underline{X}_{D \backslash k} \leftarrow X_S),$$
where $S$ is the set of features that appears earlier in the order $\pi$. Although such an order based decomposition is additive, it has a major disadvantage: The attribution of relevance to the variables is strictly prioritized. Variables that appear earlier in the order are fully attributed with the contribution that they share with variables that appear later in the order. As such, features that contain prediction-relevant information that enters via $\underline{X}_k$ but that share their full contribution with variables that appear earlier in the order, are deemed irrelevant. Additionally, the contribution that two variables jointly achieve is fully attributed to the variable that appears later in the order.
Consequently, since sensitivity is prioritized, $\beta_{kj}$ is to be preferred over $\beta_{kj}^{\pi}$.\\
\\
\textbf{Conditional SAGE:} SAGE \cite{covert_understanding_2020} is based on a linear combination of value functions $v$. In order to compute the surplus contribution of variable $j$ given the set of features $C$, the difference in value with and without the respective variable is computed. In our notation, these differences can be written as
$$\alpha_j^C = v(C \cup j) - v(C) = AI(X_j|X_S; f_C, f_{C \cup i}).$$
Therefore, we can leverage \textit{AI via} to compute the component of the associative importance that enters the model via feature $\underline{X}_k$. If we were to introduce the feature in isolation, interactions between $\underline{X}_k$ and other features cannot be restored as they are left perturbed. As we prioritize sensitivity, we instead compare the total importance with the importance, when only the feature of interest is removed. The contribution of the interaction is then attributed to every partaking feature.
$$\alpha_{jk}^C = \alpha_j^C - AI(X_j|X_C \rightarrow \underline{X}_{D \backslash k}; f_C, f_{C \cup j})$$
The SAGE attribution $\phi_j$ for feature $X_j$ is a linear combination of value function evaluations $\alpha_j^C$. We yield the component of the SAGE value that can be attributed to feature $\underline{X}_k$ by replacing the difference terms $\alpha_j^C$ with the respective $\alpha_{jk}^C$ values that only account for the contribution of feature $\underline{X}_k$.\\
\\
\textbf{Computational complexity:} If the decomposition is to be performed for all $d$ features/variables, the decomposition scales in $\mathcal{O}(d^2)$. It is possible to decompose subsets of the features/variables into the components corresponding to a subset of the features/variables. I.e., in order to decompose the relevance of one feature/variable with respect to $m$ components, the number of additional \textit{DI from}/\textit{AI via} evaluations scales in $\mathcal{O}(m)$. 
\subsection{DEDACT: Shapley Based Decomposition}
\label{subsec:shapley-decomposition}
In order to achieve an additive decomposition of associative and direct importance, the game theoretic Shapley values \cite{shapley1953value} can be adduced. As some recent work elaborates \cite{Datta2017,lundberg_unified_2017,covert_understanding_2020}, attribution of importance can be seen as a cooperative game in which players (features/variables) collaborate to yield a certain payoff (importance). As the surplus contribution of a player depends on which players are already in the room, the Shapley value averages the surplus contribution of the feature with respect to all possible configurations. Shapley values yield the unique attribution that satisfies a number of fairness axioms \cite{shapley1953value} (Appendix \ref{appendix:shapley}).\\
In order to apply Shapley values to decompose direct and associative importance measures, we need to define the corresponding value functions $w$. Then, the attribution $\phi_i$ is defined as
$$ \phi_i(w) = \frac{1}{d} \sum_{S \subseteq D \backslash \{i\}} \binom{d-1}{|S|}^{-1} [w(S \cup \{i\}) - w(S)].$$
\textbf{Permutation Feature Importance:} In order to formulate the decomposition of Permutation Feature Importance as a cooperative game, we leverage \textit{DI from} to quantify the payoff for a team of players $X_J$.\\
$$ w_k(J) := DI(\underline{X}_k|\underline{X}_{D \backslash k} \leftarrow X_J; f, f)$$
\textbf{Conditional SAGE:} The decomposition of the SAGE value function $v$ itself can be formulated as a cooperative game. The corresponding value function $w$ is defined as 
$$ w_j^C(K) := AI(X_j|X_C \rightarrow \underline{X}_K; f_C, f_{C \cup j})$$
\textbf{Computational complexity:} As the number of possible subsets grows exponentially in the number of features, the computation of Shapley values is expensive. Like previous work in the field, we suggest to approximate Shapley values by randomly sampling and evaluating orders until the estimates converge \cite{lundberg_unified_2017,covert_understanding_2020}.

\section{Simulations}
\label{sec:application}
We illustrate the usefulness of the approach on two simulated examples. Both provide access to the ground-truth causal graphs, which allows us to validate the interpretation. Furthermore, we chose all relationships to be linear with Gaussian noise, because in this setting reliable conditional distribution estimation is readily available.\footnote{The code for our experiments is available in an anonoymized repository \href{https://github.com/anonomyzed-submission/dedact-code}{[click here]}.}
%
\begin{example}[Biomarker Failure]
We consider a \textit{prostate cancer} ($Y$) diagnosis setting. For the model training and evaluation, we don't have access to the true $Y$, but only to labels $L$ that are wrongfully influenced by \textit{PSA} $P$.\footnote{The \textit{prostate specific antigen (PSA)} was wrongfully used in clinical practice for several decades \cite{ioannidis_biomarker_2013}, leading to over-treatment and its retraction from clinical usage.} In order to avoid bias from \textit{PSA}, the model is fit with access to two variables only: \textit{biomarker} $B$ and \textit{cycling} $C$. However, the model learns to use \textit{cycling} as a proxy for \textit{PSA},\footnote{Studies suggest that cycling increases PSA levels \cite{jiandani2015effect}.} such that the bias leaks into the model. A fast DEDACT decomposition of the associative influence \textit{PSA} exposes that the variable leaks into the model via the dependent \textit{cycling} feature (Figure \ref{fig:psa-results} (a)). Further, a fast DEDACT decomposition of the PFI of the feature \textit{cycling} shows that it's importance can be fully attributed to \textit{PSA}, suggesting the removal of the feature from the model.
\label{example:psa-di-vs-ii}
\end{example}
\begin{figure}[]
    \centering
    \input{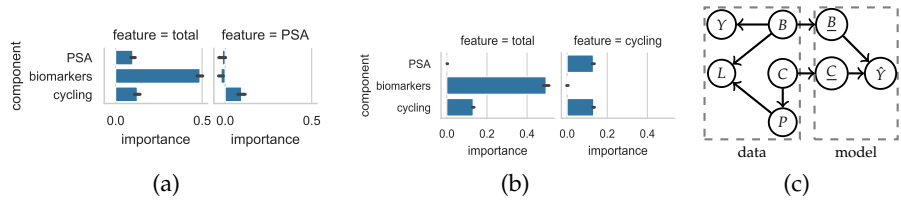}
\caption{(a) Associative importance $AI(X_j; f, f)$ and the respective fast decomposition for $PSA$. (b) Direct importance (PFI) and the respective fast decomposition for \textit{cycling}. (c) The DAG of the ground-truth Structural Causal Model. \textit{PSA} ($P$) is causal for the historical labels, although only the \textit{biomarker} ($B$) is relevant for the actual condition. The model uses \textit{cycling} ($C$) as a proxy, thereby enabling the bias to leak into the model.}
\label{fig:psa-results} 
\end{figure}
%
%
\begin{example}[Sensitive Attributes]
In this hypothetical scenario, adapted from the census income dataset from the UCI Repository \cite{Dua:2019}, the aim is the to predict the income of subjects. The protected attributes \textit{age}, \textit{sex} and \textit{race} affect the income of subjects directly and indirectly. \textit{Age} only affects income via \textit{capital gain}, \textit{number of educations} and \textit{hours per week}. \textit{Race} causes income directly as well as indirectly via \textit{marriage status} and \textit{occupation}. And \textit{sex} influences the target directly as well as indirectly via \textit{relationship} and \textit{work class}. The corresponding causal graph is depicted in Figure \ref{fig:census-income}.\footnote{We emphasize that this scenario is hypothetical and do not want to suggest that the postulated causal structure reflects the real world.}\\
The Shapley based DEDACT decomposition of SAGE values for features \textit{race}, \textit{sex} and \textit{age} correctly identifies the features, via which the information from the respective variable enters the model (Figure \ref{fig:ii-adult-sage}). The Shapley based DEDACT decomposition of PFI correctly identifies the respective sources of prediction-relevant information (Figure \ref{fig:ii-adult-viafrom}).
\end{example}
\begin{figure}[H]
    \centering
    \includegraphics[width=\linewidth]{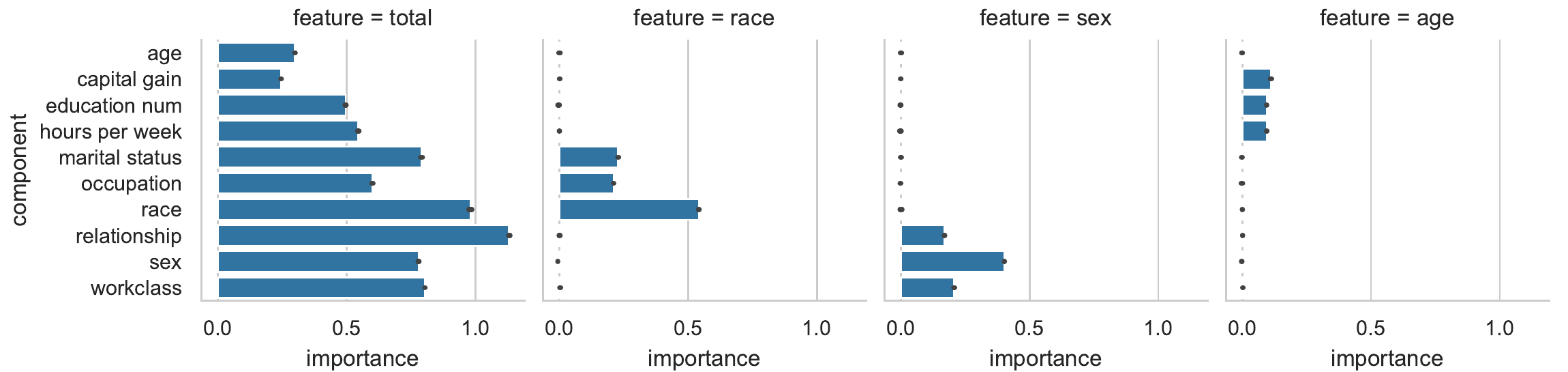}
    \caption{Approximation of the Shapley based decomposition of SAGE values ($60$ SAGE orders, $25$ decomposition orders). The procedure correctly identifies the feature pathways which allow the associative importance to enter the model.}
    \label{fig:ii-adult-sage}
\end{figure}
\begin{figure}[H]
    \centering
    \includegraphics[width=\linewidth]{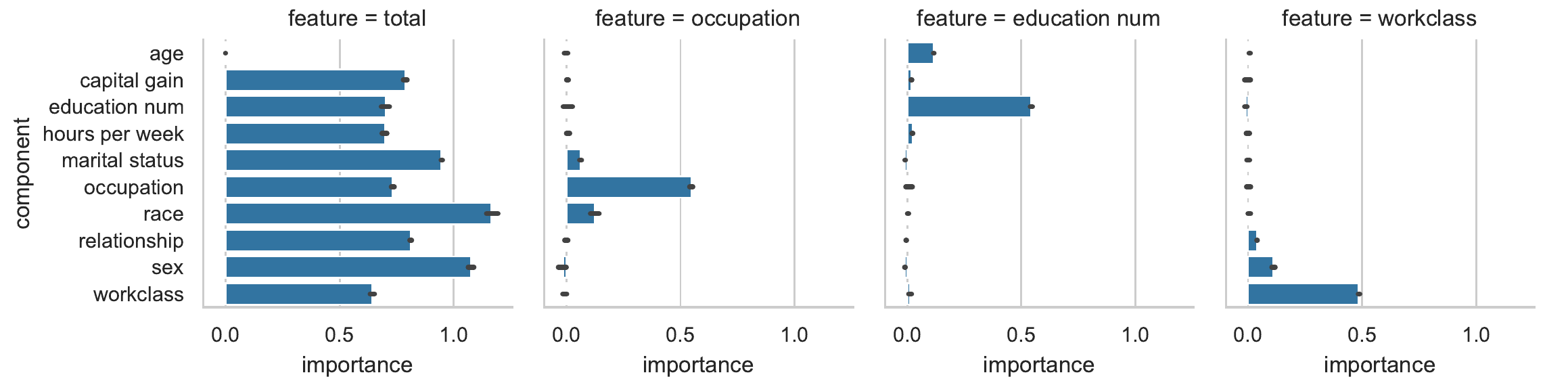}
    \caption{Shapley decomposition of PFI ($50$ orders). The sources of prediction relevant information for the features \textit{education num}, \textit{workclass} and \textit{occupation} are correctly identified.}
    \label{fig:ii-adult-viafrom}
\end{figure}
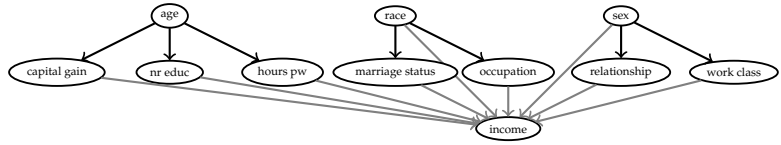
\begin{figure}[]
    \centering
    \begin{tikzpicture}[thick, scale=1.5, every node/.style={scale=.7, line width=0.25mm, black, fill=white}]
        \usetikzlibrary{shapes}
		\node[draw, ellipse, scale=0.7] (sex) at (2, .5) {sex};
		\node[draw, ellipse, scale=0.7] (age) at (-2, .5) {age};
		\node[draw, ellipse, scale=0.7] (race) at (0,.5) {race};
		\node[draw, ellipse, scale=0.7] (inc) at (1,-.5) {income};
	
	    \draw[->,gray] (race) -- (inc);
        \draw[->,gray] (sex) -- (inc);
	
		\node[draw, ellipse, scale=0.7] (hpw) at (-1,0) {hours pw};
		\node[draw, ellipse, scale=0.7] (ne) at (-2,0) {nr educ};
		\node[draw, ellipse, scale=0.7] (wc) at (3,0) {work class};
		\node[draw, ellipse, scale=0.7] (oc) at (1,0) {occupation};
		\node[draw, ellipse, scale=0.7] (cg) at (-3,0) {capital gain};
		\node[draw, ellipse, scale=0.7] (ms) at (0,0) {marriage status};
		\node[draw, ellipse, scale=0.7] (rel) at (2,0) {relationship};
		
		\draw[->] (sex) -- (rel);
		\draw[->] (sex) -- (wc);
		\draw[->] (age) -- (hpw);
		\draw[->] (age) -- (cg);
		\draw[->] (age) -- (ne);
		\draw[->] (race) -- (ms);
		\draw[->] (race) -- (oc);

        \draw[->, gray] (hpw) -- (inc);
        \draw[->, gray] (ne) -- (inc);
        \draw[->, gray] (wc) -- (inc);
        \draw[->, gray] (oc) -- (inc);
        \draw[->, gray] (cg) -- (inc);
        \draw[->, gray] (ms) -- (inc);
        \draw[->, gray] (rel) -- (inc);
        
    \end{tikzpicture}
    \caption{Ground truth DAG for the simulated adult dataset. Gray edges indicate parent edges for predicted income.}
    \label{fig:census-income}
\end{figure}
\section{Discussion}

We propose DEDACT, a general framework that allows to decompose associative importance measures into their direct feature pathway contributions as well as direct importance measures into the contributions of their information sources. Thereby, DEDACT enables novel insight into model and data. DEDACT does not only uncover leakage of information but also explains via which feature pathways the influence enters the model.\\
A fast but non-additive, as well as an additive and fair, but computationally demanding decomposition have been proposed. The approach relies on conditional samples that established methods like conditional SAGE \cite{covert_understanding_2020} require to compute anyway. Thus, the computational overhead can be reduced.\\
Further research is required evaluating how well the method scales to more complex distribution and model types. Moreover, the question of how mixed perturbations with different conditioning sets can be coupled to reduce extrapolation is a promising avenue for future research. We conjecture that methods on information preserving independent representation learning could be leveraged. In addition, future work will assess how knowledge about the dependence structure in the graph can be used to reduce the computational complexity of the method. We see DEDACT as a step towards understanding machine learning methods within the context they operate in: the underlying data generating mechanism.\\

 \bibliographystyle{splncs04}
 \bibliography{root.bib}

\appendix
\newpage

\section{Proofs}
\label{appendix:proofs}

\begin{proof}[Proof of Direct Sensitivity of DI, Proposition \ref{proposition:tdr-direct-sensitivity}]
In order to simplify notation, we define $R := D \backslash (B \cup K)$.\\
We prove the statement in two steps: First we show that if $\underline{X}_K$ is not causal for $\hat{Y}$ (over the distribution $P(X_B)P(X_{K})P(X_R)$) then $\hat{Y}$ is invariant with respect to any value that may be assigned to $\underline{X}_K$ for the evaluation of direct importance (and therefore invariant to reconstruction/destruction). Secondly, we show that given the invariance the measures evaluates to zero.\\
If the invariance holds over the distribution $P(X_B)P(X_K)P(X_R)$, then it also holds for the distributions over which the model is evaluated. Therefore we show that $P(X_B, \tilde{X}_{K}^\emptyset, \tilde{X}_{R}^\emptyset) > 0$ and $P(X_B, X_K, \tilde{X}_{R}^\emptyset) > 0$ both imply that $P(X_B)P(X_K)P(X_R) > 0$. We factorize the the distributions over which the model is evaluated before and after restoring $X_K$.
\begin{align*}
    P(X_B, \tilde{X}_{K}^\emptyset, \tilde{X}_R^\emptyset) &= P(X_B)P(\tilde{X}_{K}^\emptyset, \tilde{X}_R^\emptyset | X_B)\\
    &= P(X_B)P(\tilde{X}_{K}^\emptyset, \tilde{X}_R^\emptyset)\\
    &= P(X_B)P(X_K, X_R)
\end{align*}
\begin{align*}
    P(X_B, X_K, \tilde{X}_{R}) &= P(X_B)P(X_K, \tilde{X}_R^\emptyset | X_B)\\
    &= P(X_B)P(X_K|X_B)P(X_R)
\end{align*}
Using e.g. the law of total probability, we can see that nonzero probability of a data point on at least one of the evaluation distributions implies nonzero probability over $P(X_B)P(X_K)P(X_R)$.\\
Consequently, reconstructing $\underline{X}_K$ does not have an effect on the prediction of the non-marginalized function $f$.\\
Similarly, due to the invariance, for any $(x_B, x_K)$ with $P(X_B=x_B, X_K=x_K) > 0$ it holds that
\begin{align*}
    f_{B \cup K}(x_B, x_K) &= \mathbb{E}[f(\underline{X})^{do(\underline{X}_R^\emptyset)}|\underline{X}_B=x_B, \underline{X}_K=x_K]\\
    &= \mathbb{E}[f(\underline{X})^{do(\underline{X}_{K \cup R}=\tilde{X}_{K \cup R}^\emptyset)}|\underline{X}_B=x_B]\\
    &= f_B(x_B)
\end{align*}
Consequently, for $g$,$g'$ being either the original or marginalized prediction functions, 
\begin{align*}
    \hat{Y}_{g}|{do(\underline{X}_{K \cup R}=\tilde{X}_{K \cup R}^\emptyset)}
    &= \hat{Y}_{g'}|{do(\underline{X}_{R} = \tilde{X}_{R}^\emptyset)}.
\end{align*}
As follows, 
\begin{align*}
    DI(\underline{X}_K|\underline{X}_B;g,g') := &\mathbb{E}[\loss(\hat{Y}_{g}|{do(\underline{X}_{K \cup R }=\tilde{X}_{K \cup R}^\emptyset)};Y)]\\ 
    - &\mathbb{E}[\loss(\hat{Y}_{g'}|{do(\underline{X}_{R} = \tilde{X}_{R}^\emptyset)};Y)]\\
    = &0.
\end{align*}
Contraposition proves the statement.
\end{proof}
\begin{proof}[Proof of Associative Sensitivity of AI, Proposition \ref{proposition:tar-associative-sensitivity}]\\
\textit{Non-marginalized functions (Case 1):}
In order to prove the first statement, we show that under conditional independence the interventional distributions with respect to $C$ and $C \cup J$ coincide. As a consequence, associative importance with respect to the non-marginalized functions evaluates to zero. More specifically, if the conditional independence $Y \idp X_J | X_C$ holds, then 
\begin{align*}
    P(Y, \tilde{X}^C, X_C) &= P(Y|\tilde{X}^C, X_C) P(\tilde{X}^C, X_C)\\ 
    &= P(Y|\tilde{X}^C, X_C, X_J) P(\tilde{X}^C, X_C)\\
    &= P(Y|\tilde{X}^{C \cup J}, X_C, X_J) P(\tilde{X}^{C\cup J}, X_C)\\
    &= P(Y, X_C, \tilde{X}^{C \cup J})
\end{align*}
As follows, 
$$P(\tilde{X}^C, Y) = P(\tilde{X}^{C \cup J}, Y)$$
and therefore 
$$P(Y,\hat{Y}_f|{do(\underline{X}_{D\backslash C} = \tilde{X}_{D\backslash C}^{C})}) = P(Y,\hat{Y}_{f}|{do(\underline{X}_{D\backslash (C \cup J)} = \tilde{X}_{D \backslash (C \cup J)}^{C \cup J})}).$$
and consequently the measure $AI(X_J|X_C;f,f)$ evaluates to zero. The respective contraposition proves the statement.\\
\textit{Marginalized functions (Case 2):}
If the marginalized functions $f_C, f_{C \cup J}$ are used, further assumptions are required. To be more precise, we prove that given mean squared error or cross entropy loss and the corresponding loss optimal predictor, conditional independence implies zero associative importance. Contraposition proves the statement.\\
For our proofs we rely on results for SAGE value functions, detailed in \cite{covert_understanding_2020}, Appendix C. Therefore, we rewrite associative importance given marginalized functions in terms of SAGE value functions:%
\begin{align*}
AI(X_J|X_C;f_C, f_{C \cup J}) = &\mathbb{E}[\loss(Y; \hat{Y}_{f_C}|{do(\underline{X}_{D\backslash C} = \tilde{X}_{D\backslash C}^{C})})]\\
- &\mathbb{E}[\loss(Y;\hat{Y}_{f_{C \cup J}}|{do(\underline{X}_{D\backslash (C \cup J)} = \tilde{X}_{D \backslash (C \cup J)}^{C \cup J})})]\\
= &v_{f_{C \cup J}}(C \cup J) - v_{f_C}(C) 
\end{align*}
\textit{Mean Squared Error:} Given mean squared error as loss and the corresponding loss optimal predictor $f^*$, it holds that:
$$v_{f_{C \cup J}^*}(C \cup J) - v_{f_C^*}(C) = \mathbb{E}[\text{Var}(Y|X_C)] - \mathbb{E}[\text{Var}(Y|X_{C \cup J})]$$
Under conditional independence $Y \idp X_J | X_C$ it follows that
\begin{align*}
\mathbb{E}[\text{Var}(Y|X_{C \cup J})] &= \mathbb{E}[\mathbb{E}[\text{Var}(Y|X_{C \cup J})|X_C]]\\
&= \mathbb{E}[\text{Var}(Y|X_C)]
\end{align*}
and consequently $Y \idp X_J | X_C \Rightarrow AI(X_J|X_C;f^*_C, f^*_{C \cup J}) = 0$.\\
\textit{Cross Entropy:} Given cross entropy as loss and the corresponding loss optimal predictor $f^*$, it holds that:
$$v_{f_{C \cup J}^*}(C \cup J) - v_{f_C^*}(C) = I(Y;X_J|X_C)$$
Mutual information $I(Y;X_J|X_C)$ is zero if and only if $Y \idp X_J | X_C$. Consequently $AI(X_J|X_C;f^*_C, f^*_{C \cup J}) = 0 \Leftrightarrow X_J \idp Y | X_C$.\\
\end{proof}



\section{Shapley Value Axioms}
\label{appendix:shapley}

Shapley values are the unique solution that satisfies a set of fairness axioms \cite{shapley1953value}. Given value function $v$, the attributions $\phi$ satisfy the following requirements:\\
\begin{enumerate}
	\item Efficiency: $\sum_{i=1}^d \phi_i(v) = v(D) - v(\emptyset)$
	\item Symmetry: $v(S \cup \{i\}) = v(S \cup \{k\}$ for all S then $\alpha_i(v) = \phi_j(v)$
	\item Linearity: $v(S) = \sum_{k=1}^n c_k v_k (S)$ which is a linear combination of games $(v_1, \dots, v_n)$ has scores $\phi_i(v) = \sum_{k=1}^n c_k \phi_i (v_k)$
	\item Monotonicity: If for two games $v, v'$ we have $v(S \cup \{i\})-v(S) \geq v'(S\cup \{i\})-v'(S)$ for all $S$, then $\phi_i(v) \geq \phi_i(v')$
	\item Dummy: If $v(S) = v(S \cup \{i\})$ for all $S$, then $\phi_i(v) = 0$
\end{enumerate}
\end{document}